\documentclass{article}

\usepackage{microtype}
\usepackage{graphicx}
\usepackage{subfigure}
\usepackage{booktabs} 
\usepackage{mathrsfs}

\usepackage{thm-restate}

\usepackage{nicefrac}
\usepackage{xcolor}
\usepackage{amsmath,amssymb,amsthm}

\theoremstyle{definition}
\newtheorem{theorem}{Theorem}
\newtheorem{proposition}[theorem]{Proposition}

\newtheorem{lemma}[theorem]{Lemma}

\usepackage{enumitem}
\setlist[enumerate]{leftmargin=0.5cm,topsep=0pt,itemsep=-2pt}
\setlist[itemize]{leftmargin=0.5cm,topsep=0pt,itemsep=-2pt}

\usepackage{mathrsfs}

\usepackage{mathtools}

\usepackage{hyperref}
\hypersetup{
    colorlinks=true,
    linkcolor=blue,
    citecolor=cyan,
}

\newcommand{\piref}{\pi_{\text{ref}}}
\newcommand{\E}{\mathbb{E}}

\newcommand{\eqdef}{\stackrel{\small{\mathsf{def}}}{=}}
\newcommand{\KL}{\mathbb{KL}}

\usepackage[accepted]{icml2021}

\icmltitlerunning{RL-finetuning LLMs from on- and off-policy data with a single algorithm}

\begin{document}

\twocolumn[
\icmltitle{RL-finetuning LLMs from on- and off-policy data with a single algorithm}

\icmlsetsymbol{equal}{*}

\begin{icmlauthorlist}
\icmlauthor{Yunhao Tang}{meta1,equal}
\icmlauthor{Taco Cohen}{meta2}
\icmlauthor{David W. Zhang}{meta2}
\icmlauthor{Michal Valko}{old_meta}
\icmlauthor{R\'emi Munos}{meta2,equal}

\end{icmlauthorlist}
\icmlaffiliation{meta1}{Meta GenAI}
\icmlaffiliation{meta2}{Meta FAIR}
\icmlaffiliation{old_meta}{Work done while at Meta}

\icmlkeywords{Machine Learning, ICML}

\vskip 0.3in
]

\printAffiliationsAndNotice{\icmlEqualContribution} 
\begin{abstract}

We introduce a novel reinforcement learning algorithm (AGRO, for Any-Generation Reward Optimization) for fine-tuning large-language models. AGRO leverages the concept of generation consistency, which states that the optimal policy satisfies the notion of consistency across any possible generation of the model. We derive algorithms that find optimal solutions via the sample-based policy gradient and provide theoretical guarantees on their convergence. Our experiments demonstrate the effectiveness of AGRO in both on-policy and off-policy settings, showing improved performance on the mathematical reasoning dataset over baseline algorithms.

\end{abstract}

\section{Introduction}

Large language models have displayed powerful capabilities, and upon fine-tuning, can become helpful assistants to human users. Although supervised learning has been a dominant fine-tuning paradigm due to its efficacy and stability, reinforcement learning from human feedback (RLHF) has played an increasingly important role \citep{christiano2017deep,ziegler2019fine,nakano2021webgpt,ouyang2022training,bai2022constitutional}. Although previously it was commonly believed that reinforcement learning (RL) only carries out lightweight fine-tuning on top of an existing (already fine-tuned) model, it has also been proven to be highly effective in entailing new capabilities beyond previous pre-training and post-training paradigms \citep{jaech2024openai,guo2025deepseek,lambert2024t,team2025kimi}.

The existing RL algorithms for language model fine-tuning can be largely categorized into two groups: online and offline. Online RL algorithms consist of iterative sampling and updating of the latest policy model, assuming access to an external reward model that provides the learning signal \citep{ziegler2019fine,ouyang2022training,munos2023nash,guo2024direct,calandriello2024human}. Meanwhile, offline RL algorithms directly extract information from a static offline dataset, bypassing the need for iterative sampling and potentially reward modeling \citep{zhao2023slic,rafailov2024direct,azar2024general,tang2024generalized,richemond2024offline}. There is a clear trade-off between the two classes of algorithms: while offline algorithms are much cheaper to execute, they generally deliver less improvement compared to online algorithms \citep{xu2024dpo,tang2024understanding,tajwar2024preference}.

The dichotomy between online and offline algorithms also makes it difficult to directly adapt algorithms from one case to another. Most importantly, we will see that naively adapting the on-policy KL regularized policy gradient algorithm - one of the most prominent on-policy RL algorithms - to the off-policy case, will not lead to the optimal target policy (Figure~\ref{figure:tabular}). The natural motive is to build an algorithm that works seamlessly for both on-policy and off-policy cases.

In this work, we propose a novel RL algorithm called Any-Generation Reward Optimization (AGRO) for fine-tuning language models. Central to AGRO is the insight that the optimal policy for the RLHF problem satisfies a notion of \emph{generation consistency}, which we will detail later. This consistency is rooted in the regularized policy optimization that underlies the RLHF formulation. This notion of consistency allows us to derive algorithms that leverage both on-policy and off-policy data, and find optimal policy via sample-based learning with theoretical guarantees. Our technical contributions are detailed as follows.
\begin{itemize}
    \item We identify generation consistency (Section~\ref{sec:consistency}), a notion of optimality condition inherent to the regularized policy optimization problem. Given this, we propose a few loss functions, with which we can search for the optimal policy. These losses are also conducive to sample-based gradient estimations (Section~\ref{sec:gradient}).
    \item We propose Any-Generation Reward Optimization (AGRO, Section~\ref{sec:agro}), which is derived from the consistency condition and can leverage both on-policy and off-policy data for policy optimization. We also discuss subtlety in implementing AGRO at the token level for LLM applications (Section~\ref{sec:token.level.gradient.estimate}).
    \item We conclude with experimental ablations that showcase the competitive performance of AGRO in both on-policy and off-policy learning settings, with a focus on the mathematical reasoning domain (Section~\ref{sec:exp}). We also ablate on important hyper-parameters that produce actionable insights for practitioners.
\end{itemize}

\section{Reinforcement learning for language model}

A language model can be understood as a policy $\pi$ in the context of RL \citep{sutton1998}. Given a prompt $x\in \mathcal{X}$, the policy produces a response (or generation) $y\in \mathcal{Y}$, which then gets evaluated by a certain reward function $r(x,y)$, which usually reflects human preference. The objective is to optimize $\pi$ such that that the expected reward is maximized: depending on the application, the reward function can be a parametric network \citep{christiano2017deep,ziegler2019fine,ouyang2022training} or programmatic \citep{lightman2023let,uesato2022solving,gehring2024rlef}. Formally, consider the regularized policy optimization problem to maximize the regularized objective 
\begin{align}
    {\cal G}(\pi)\eqdef \E_{x\sim\rho}\left[\E_{y\sim \pi(\cdot|x)}\left[ r(x,y)\right] - \beta \KL(\pi(\cdot|x), \piref(\cdot|x))\right].
\label{eq:regularized.objective}
\end{align}
By construction, the objective aims to maximize the average reward subject to the KL regularization penalty $\mathbb{KL}\left(\pi(\cdot|x),\piref(\cdot|x)\right)\eqdef\mathbb{E}_{y\sim \pi(\cdot|x)}\left[\log \frac{\pi(y|x)}{\piref(y|x)}\right]$ that encourages $\pi$ to stay close to some reference policy $\piref$ during optimization. Here, $\rho$ denotes the sampling distribution over the space of prompt $X$; the coefficient $\beta\geq 0$ determines the trade-off between the policy optimization and the regularization. The reference policy $\piref$ is usually the supervised finetuned policy and the starting point of the optimization.

\section{Generation consistency for regularized RL}\label{sec:consistency}

We start with the key observation that a notion of consistency condition is inherent to the optimal solution of the regularized RL (RLHF) problem. Formally, this result can be stated as follows.
\begin{theorem} (\textbf{Generation consistency at optimality}) \label{thm:consistency}
Let $\pi^* \eqdef \arg\max_\pi {\cal G}(\pi)$ be the optimal policy to the RLHF problem defined by Eq.~\eqref{eq:regularized.objective}. 
Then for any generation $y$, the following quantity computed per $(x,y)$,
\begin{align*}
    r(x,y) - \beta \log \frac{\pi^\ast(y|x)}{\piref(y|x)}
\end{align*}
does not depend on $y$. It is a function of the prompt $x$ only.
\end{theorem}
\begin{proof}
    It is well-known that the optimal policy has an analytic form \citep{rafailov2024direct,calandriello2024human,richemond2024offline}: for all $x,y$,
\begin{equation}\label{eq:pi.star}
\pi^*(y|x) = \frac{\piref(y|x) e^{\frac 1\beta r(x,y)}}{e^{\frac 1\beta \widetilde{V}^{\pi^*}(x)}},
\end{equation}
where the normalizing constant is
$$\widetilde{V}^{\pi^*}(x)\eqdef \beta \log \Big( \sum_{y\in \mathcal{Y}} \piref(y|x) e^{\frac 1\beta r(x,y)}\Big).$$ 
We can verify that for any $y$, $r(x,y) - \beta \log \frac{\pi^\ast(y|x)}{\piref(y|x)}=\widetilde{V}^{\pi^\ast}(x)$, which is a function of the prompt $x$ only.
\end{proof} 

\paragraph{Remark.} We can also introduce the traditional value function which is defined as the expected reward $V^{\pi}(x) \eqdef \sum_{y\in \mathcal{Y}} \pi(y|x) r(x,y)$ \citep{sutton1999}Note that by plugging $\pi^*$ back into Equation~\eqref{eq:regularized.objective} we have the property that: 
\begin{eqnarray*}
\widetilde{V}^{\pi^*}(x) &=& V^{\pi^*}(x) - \beta \KL(\pi^*(x), \piref(x)).
\end{eqnarray*}
That is, the normalizing constant $\widetilde{V}^{\pi^*}(x)$
is indeed a KL-regularized value function \citep{ziebart2008maximum,nachum2017bridging,levine2018reinforcement}.

\begin{figure}[t]
\centering
\includegraphics[width=3in]{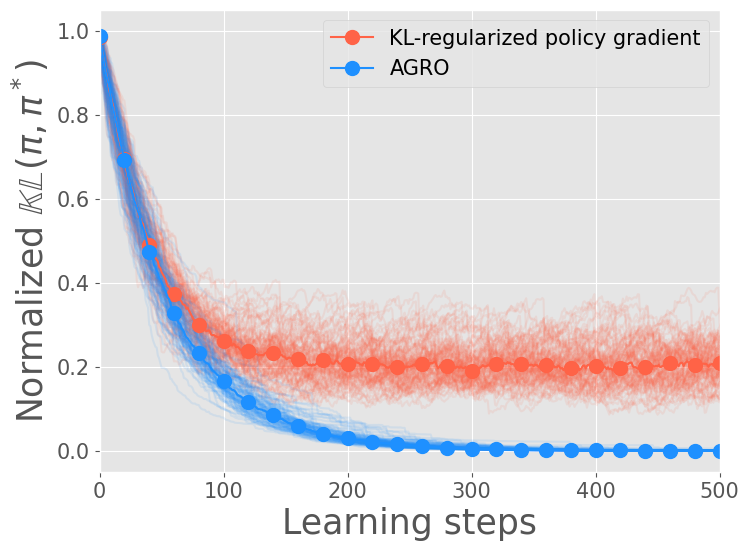}
\caption{\small{KL-regularized policy gradient vs. AGRO in the tabular case with off-policy data (data generated from $\piref$ rather than $\pi$). We measure the normalized KL-divergence  $\mathbb{KL}(\pi,\pi^*)$ between $\pi$ and the optimal policy $\pi^*$ during training. We see that under off-policy data, regularized KL-regularized policy gradient does not converge to the optimal policy $\pi^*$ while AGRO converges as evidenced by the vanishing KL divergence.}}
\label{figure:tabular}
\end{figure}

\subsection{Deriving loss functions from generation consistency}
Henceforth for simplicity, we introduce the regularized reward which depends on the policy $\pi$,
\begin{align*}
    R^\pi_\beta(x,y)\eqdef r(x,y) - \beta \log\frac{\pi(y|x)}{\piref(y|x)}.
\end{align*}
Theorem~\ref{thm:consistency} implies the following property on the regularized reward $R_\beta^{\pi^*}$ of the optimal policy: for any generation $y$, 
\begin{align}\label{eq:pi*.equation}
R^{\pi^*}_\beta(x,y) - \E_{y'\sim\mu(\cdot|x)} \left[ R^{\pi^*}_\beta(x,y') \right] = 0.
\end{align}
where $\mu(\cdot|x)$ is any distribution over $\mathcal{Y}$ that support beyond just the sample $y$. Naturally, we can design a squared loss that enforces the above equality in order to find the optimal policy, i.e., we minimize with respect to $\pi$, from every $x$,
\begin{align*}
    \mathbb{E}_{y\sim \pi(\cdot|x)}\left[\left(R^{\pi}_\beta(x,y) - \E_{y'\sim\pi(\cdot|x)} \left[ R^{\pi}_\beta(x,y') \right]\right)^2 \right].
\end{align*}
Note that if we think of $R^{\pi}_\beta(x,y)$ as a random variable, the above can be understood as its variance under the distribution $y\sim \mu(\cdot|x)$. As a result, the loss can simplify as 
\begin{align*}
     {\mathbb V}_{y\sim\mu(\cdot|x)}\left(R^{\pi}_\beta(x,y)\right).
\end{align*}

Finally, we can define the aggregate loss as the expectation (over $x$). We start with the on-policy loss where we let $\mu=\pi$, and define the 
\emph{on-policy loss}: 
$${\cal L}(\pi) \eqdef \frac 12 
{\E}_{x\sim\rho}\left[ {\mathbb V}_{y\sim\pi(\cdot|x)}\left(R^{\pi}_\beta(x,y)\right)\right].
$$
Note that both the regularized reward $R^{\pi}_\beta$ and the sampling depend on $\pi$. Similarly, in the off-policy case (e.g., responses $y$ have been generated by some behavior policy $\mu$), e.g. when using a fixed dataset or a replay buffer, we define the \emph{off-policy loss} where $\mu$ has no optimization dependency on $\pi$:
$${\cal L}_\mu(\pi) \eqdef \frac 12 
{\E}_{x\sim\rho}\left[ {\mathbb V}_{y\sim\mu(\cdot|x)}\left(R^{\pi}_\beta(x,y)\right)\right].
$$

Importantly, we show that globally minimizing both losses will lead to the optimal policy.

\begin{theorem} (\textbf{The unique global minimizer is the optimal policy})
Consider the space $\Pi$ of policies that have the same support as $\piref(\cdot|x)$ on all states $x\in \mbox{support}(\rho)$. Assume $\mu\in\Pi$. Then $\pi^*$ is the unique global minimum  in $\Pi$ of both ${\cal L}(\pi)$ and ${\cal L}_\mu(\pi)$ on $\mbox{support}(\rho)$.  
\end{theorem}
\begin{proof}
We provide the proof here as the argument is quite constructive. First notice that ${\cal L}(\pi)\geq 0$ and ${\cal L}_{\mu}(\pi)\geq 0$, and from \eqref{eq:pi*.equation} we have that ${\cal L}(\pi^*)= 0$ and ${\cal L}_{\mu}(\pi^*)=0$. Thus $\pi^*$ is a global minimum of both ${\cal L}(\pi)$ and ${\cal L}_\mu(\pi)$.

Now, let's prove that $\pi^*$ is the unique global minimum of both ${\cal L}(\pi)$ and ${\cal L}_\mu(\pi)$. The losses  ${\cal L}(\pi)$ and ${\cal L}_{\mu}(\pi)$ are the expectation (under $x\sim\rho$) of the variance 5of $R^\pi_\beta(x,y)$ when $y\sim\pi(\cdot|x)$ (respectively when $y\sim\mu(\cdot|x)$).

For any policy $\pi\in\Pi$, for any $x\in\mbox{support}(\rho)$ we have that $\pi(\cdot|x)$ has the same support as $\piref(\cdot|x)$, which, from Eq.~\ref{eq:pi.star}, is the same as the support of $\pi^*(\cdot|x)$. Thus, for the on-policy loss, for any $x\in\mbox{support}(\rho)$, this variance is zero if and only if for any $y\in \mbox{support}(\pi(\cdot|x))$
(respectively $y\in \mbox{support}(\mu(\cdot|x))$ for the off-policy loss), the term $R^\pi_\beta(x,y)$ is independent of $y$ (but can be a function of $x$), which is equivalent to say that 
$$\pi(y|x)\propto \piref(y|x) e^{\frac 1\beta r(x,y)}.$$
From \eqref{eq:pi.star} we deduce that $\pi(y|x)=\pi^*(y|x)$. Thus $\pi^*$ is the unique global minimum in $\Pi$ of both ${\cal L}(\pi)$ and ${\cal L}_{\mu}(\pi)$ on the support of $\rho$.   
\end{proof}

\section{Gradient of the consistency losses}\label{sec:gradient}

We now analyze the gradient of the loss functions above.

\subsection{Gradient of the off-policy loss}\label{sec:off.policy.loss}
Throughout, we consider a parametric policy $\pi$ and let $\nabla$ denote the gradient with respect to its parameters. We start with the gradient of the off-policy loss ${\cal L}_\mu(\pi)$.
\begin{lemma} (\textbf{Off-policy gradient}) \label{lemma:off-policy-grad}
The gradient of the off-policy loss ${\cal L}_{\mu}(\pi)$ with respect to the policy parameter is:
\begin{align*}
\nabla {\cal L}_\mu(\pi) &= 
- \beta 
\E
\left[ \left(R^\pi_\beta(x,y)-R^\pi(x,\mu)\right)  \nabla\log \pi(y|x) \right],
\end{align*}
\end{lemma}
where the sampling is $x\sim\rho, y\sim \mu(\cdot|x)$. Henceforth, we detail proof of theoretical results in Appendix~\ref{sec:appx.proofs}.
Notice that we have subtracted the \emph{baseline} 
\begin{align*}
R^\pi_\beta(x,\mu)\eqdef \E_{y'\sim\mu(\cdot|x)} \big[ R^\pi_\beta(x,y') \big]
\end{align*}
 from the regularized reward $R^\pi_\beta(x,y)$. This baseline term stems from the definition of the objective ${\cal L}_\mu(\pi)$ itself, and we caution that it is not zero-mean in general
\begin{align*}
\E_{y\sim\mu(\cdot|x)} \left[
R^\pi_\beta(x,\mu)
\nabla \log\pi(y|x)
\right]  \neq 0,
\end{align*}
unless the sampling is on-policy $\mu=\pi$, in which case it can be understood as a control variate for variance reduction.

\paragraph{Equivalence to RLHF gradient when on-policy.}
We can interpret the above gradient as updating with respect to the regularized reward $R_\beta^\pi(x,y)$. In fact,  we can further break down the gradient for prompt $x$ as
\begin{align*}
    -\beta \mathbb{E}_{y\sim \mu(\cdot|x)}\!\left[r(x,y)\nabla \log \pi(y|x) -  \frac{\beta}{2}  \nabla\!\Big(\log\frac{\pi(y|x)}{\piref(y|x)}\Big)^2\right]
\end{align*}
where the first term maximizes reward while the second term minimizes the squared loss divergence between $\pi$ and $\piref$. When the sampling is on-policy $\mu=\pi$, the gradient of the squared loss is aligned with the gradient of the KL regularization \citep{calandriello2024human,tang2024generalized}. In this case, the off-policy AGRO gradient is equivalent to the RLHF gradient from Eqn~\eqref{eq:regularized.objective}.

\subsection{Gradient of the on-policy loss}
When differentiating the on-policy loss ${\cal L}(\pi)$ with respect to the policy parameter we need to take into account the fact that the sampling distribution depends on the parameter as well. Writing $f(x, y, \pi) = \frac 12 \left( 
R^\pi_\beta(x,y) - R^\pi_\beta(x,\pi)
\right)^2$, the gradient thus contains two terms: the pathwise-derivative (PD) estimate and the likelihood ratio (LR) estimate \citep{glasserman2004monte}:
\begin{align*}
\nabla {\cal L}(\pi) = \underbrace{
\E 
\left[ \nabla f(x, y, \pi) \right]
}_{\nabla_\text{PD} {\cal L}(\pi)} + \underbrace{
\E 
\left[ f(x, y, \pi)\nabla\log\pi(y|x)\right]}_{\nabla_\text{LR} {\cal L}(\pi)},
\end{align*}
where the expectation is under $x\sim \rho,y\sim \pi(\cdot|x)$. 
The pathwise-derivative $\nabla_\text{PD} {\cal L}(\pi) $ is the gradient of the off-policy loss $\nabla {\cal L_\mu}(\pi)$ when setting $\mu=\pi$, i.e. $\nabla {\cal L_\mu}(\pi)|_{\mu=\pi}$,
\begin{align*}
&= - \beta 
\E_{ \! \! \! \tiny
\begin{array}{l}
x\sim\rho \\
y\sim\pi(\cdot|x)
\end{array} 
} \! \! \! \! \! 
\left[ R^\pi_\beta(x,y) \nabla\log \pi(y|x) \right].\\
&= - \beta 
\E_{ \! \! \! \tiny
\begin{array}{l}
x\sim\rho \\
y\sim\pi(\cdot|x)
\end{array} 
} \! \! \! \! \! 
\left[ \left( R^\pi_\beta(x,y) - R^\pi_\beta(x,\pi)\right) \nabla\log \pi(y|x) \right].
\end{align*}

Now, the likelihood ratio estimate $\nabla_\text{LR} {\cal L}(\pi)$ is 
\begin{align*}
& \frac 12 
\E_{ \! \! \! \tiny
\begin{array}{l}
x\sim\rho \\
y\sim\pi(\cdot|x)
\end{array} 
} \! \! \! \! \! 
\left[ \left( 
R^\pi_\beta(x,y) - R^\pi_\beta(x,\pi)
\right)^2 \nabla\log\pi(y|x)\right].
\end{align*}
Notice again that we can substract as baseline the variance of $R^\pi_\beta(x,y)$ as a possible variance reduction technique without introducing bias, leading to this estimate of $\nabla_\text{LR} {\cal L}(\pi)$:
{\small 
\begin{align*}
\frac 12 
\E_{ \! \! \! \tiny
\begin{array}{l}
x\sim\rho \\
y\sim\pi(\cdot|x)
\end{array} 
} \! \! \! \! \! 
\left[ \left( \left( R^\pi_\beta(x,y) - R^\pi_\beta(x,\pi)
\right)^2 - b(x) \right) \nabla\log\pi(y|x)\right],
\end{align*}
} 

\vspace{-3mm}
where the baseline is $b(x)\eqdef{\mathbb V}_{y'\sim\pi(\cdot|x)} \left( R^\pi_\beta(x,y') \right) $ which we can approximate with samples. Finally, the full gradient of the on-policy loss is
\begin{align*}
\nabla {\cal L}(\pi) &= \nabla_\text{PD} {\cal L}(\pi) + \nabla_\text{LR} {\cal L}(\pi) 
\end{align*}

We provide a few remarks as follows.
\paragraph{Alignment of $\nabla_\text{PD} {\cal L}(\pi) $ with the gradient of the original objective.} Notice that the pathwise-derivative estimate is aligned with the gradient of the original policy objective ${\cal G}$:
$$\nabla_\text{PD} {\cal L}(\pi) = \nabla {\cal L_\mu}(\pi)|_{\mu=\pi} = -\beta \nabla {\cal G}(\pi).$$

\paragraph{Different learning dynamics, same optimum.} Notice that since both $\pi\mapsto {\cal G}(\pi)$ and $\pi\mapsto {\cal L}(\pi)$ have $\pi^*$ as optimum, by following the full gradient $\nabla {\cal L}(\pi)$ or the pathwise-derivative $\nabla_\text{PD} {\cal L}(\pi)$ only will both lead to the optimal policy $\pi^*$. However the learning dynamics generated by these two gradients will be different.

\paragraph{Gradient magnitude.}
    It may seem that $\nabla_\text{PD} {\cal L}(\pi)$ is negligible compared to $\nabla_\text{LR} {\cal L}(\pi)$, specially when $\beta$ is small. However, rewriting $R^\pi_\beta(x,y) = r(x,y)-\beta\log\frac{\pi(y|x)}{\piref(y|x)} = \beta\log\frac{\pi^*(y|x)}{\pi(y|x)} + \widetilde{V}^{\pi^*}(x)$, we see that $R^\pi_\beta(x,y)- R^\pi_\beta(x,\pi)=\beta\left(\log\frac{\pi^*(y|x)}{\pi(y|x)} + \KL(\pi(x), \pi^*(x)) \right)$ is of order $O(\beta)$, thus both $\nabla_\text{PD} {\cal L}(\pi)$ and $\nabla_\text{LR} {\cal L}(\pi)$ are of order $\mathcal{O}(\beta^2)$. 

\section{Any-Generation Reward Optimization}\label{sec:agro}

We introduce the Any-Generation Reward Optimization (AGRO) algorithm, which finds the optimal policy by leveraging the any generation consistency. We introduce a few variants of the algorithm.
Throughout, we assume that we can collect samples $\{x, (y_i, r_i)_{1\leq i\leq n} \}$ composed of prompts $x\sim\rho$, and for each $x$, one or several responses $y_1,\dots, y_n$ generated by some behavior policy $\mu$ (when off-policy) of by the current policy $\pi$ (when on-policy). We write the  corresponding rewards $\{ r_i = r(x,y_i)\}_{1\leq i\leq n}$.

\subsection{The off-policy AGRO algorithm}

We discuss a few special cases by varying the number of samples $n$ in the algorithm.

\paragraph{Single generation per prompt $n=1$.}
When $n=1$ we only have access to a single response per prompt. In that case one way to minimize the loss $\pi\mapsto {\cal L}_\mu(\pi)$ is to introduce a prompt-dependent function $V(x)$ and jointly minimize over $\pi$ and $V$ the loss ${\cal L}_\mu(\pi, V) $ defined as follows
$$
\E_{ \! \! \! \tiny
\begin{array}{l}
x\sim\rho \\
y\sim\mu(\cdot|x)
\end{array} 
} \! \! \! \! \! 
\left[ \left( r(x,y) - V(x) - \beta \log \frac{\pi(y|x)}{\piref(y|x)} 
\right)^2
\right].
$$
This loss has been introduced in \citep{richemond2024offline} where it was proven that the global optimum (both over $\pi$ and $V$) leads to the optimal policy $\pi^*$. The globally minimizing value function is also the normalizing constant $\widetilde{V}^{\pi^*}(x)$. In this paper we focus on the case $n\geq 2$ for which we avoid having to learn an auxiliary value function. In practice, this might be more desirable due to relative implementation simplicity.

\paragraph{Several generations per prompt.}
Assume we now have $n>1$ responses $y_1,\dots, y_n$ issued from the same prompt $x\sim\rho$. 
We do not have to learn an auxiliary value function of the prompt like in the previous paragraph. Instead we can build an unbiased estimate of the loss $\pi\mapsto {\cal L}_\mu(\pi)$ by approximating the expectation by an empirical average. 

Off-policy AGRO algorithm consists in following the gradient estimate $\nabla \widehat {\cal L}_\mu(\pi)\eqdef$
\begin{align}\label{eq:gradient.empirical.loss}
- \frac{\beta}{n} \sum_{i=1}^n \left( R^\pi_\beta(x,y_i) - \underbrace{\widehat R_{\beta, -i}^{\pi}(x,\mu)}_{\small\mbox{baseline}} \right)\nabla\log\pi(y_i|x),
\end{align}
where the baseline is defined as the leave-one-out average \citep{kool2019buy}:
$$ \widehat R_{\beta, -i}^\pi(x,\mu) \eqdef \frac{1}{n-1}\sum_{j\neq i} \left(r_j - \beta \log\frac{\pi(y_j|x)}{\piref(y_j|x)}\right).$$
It does not introduce any bias but usually helps reduce the variance of the estimate.  When the sampling is on-policy $\mu=\pi$, the above gradient estimate is akin to REINFORCE with leave-one-out baseline \citep{ahmadian2024back,shao2024deepseekmath}.

\paragraph{The specific case of two samples.}
Assume we now have $n=2$ generations $y_1$ and $y_2$ from the same prompt $x\sim\rho$. The empirical gradient $\nabla \widehat {\cal L}_\mu(\pi)$ is
\begin{align*}
-\beta \left( r_1 - r_2 -\beta \log\frac{\pi(y_1|x)\piref(y_2|x)}{\pi(y_2|x)\piref(y_1|x)} \right)\nabla\log\frac {\pi(y_1|x)}{\pi(y_2|x)}.
\end{align*}
This algorithm resembles the contrastive policy gradient algorithm of \cite{flet2024contrastive} and the IPO algorithm of \cite{azar2024general} where the preference $p(y_1\succ y_2)$ is replaced by the reward difference $r_1-r_2$.

\paragraph{Comparison of off-policy AGRO vs. KL-regularized policy gradient.} To illustrate the key property of off-policy AGRO, we compare with an adaptation of KL-regularized policy gradient algorithm to the off-policy learning case. Concretely, the $n$-sample KL-regularized policy gradient is implemented as
\begin{align*}
    \frac{1}{n}\sum_{i=1}^n \left(r_i - \bar{r}_{-i}\right) \nabla \log \pi(y_i|x) - \beta \nabla  \mathbb{KL}\left(\pi(\cdot|x),\piref(\cdot|x)\right),
\end{align*}
where $\bar{r}_{-i}$ is the leave-one-out reward baseline. We note that this gradient differs from off-policy AGRO gradient in Eqn~\eqref{eq:gradient.empirical.loss} on the regularization term as hinted earlier.

Figure~\ref{figure:tabular} compares the KL divergence between the optimized policy and optimal policy $\mathbb{KL}(\pi,\pi^*)$ over updates. We see that the regularized policy gradient algorithm generally fails to converge to $\pi^*$: the KL-regularization is not the \emph{correct} regularization in the off-policy case. On the other hand, off-policy AGRO gradient converges to $\pi^*$ as expected. Note also that the decrease in the KL divergence $\mathbb{KL}(\pi,\pi^*)$ is in fact equivalent to increase in the regularized reward ${\cal G}(\pi)$. We make this result formal in the appendix~\ref{appendix:additional-theory}.

\subsection{The on-policy AGRO algorithm}

The on-policy AGRO algorithm consists in following the gradient estimate
\begin{align*}
\nabla \widehat {\cal L}(\pi) &= \nabla_\text{PD} \widehat {\cal L}(\pi) + \nabla_\text{LR} \widehat {\cal L}(\pi),
\end{align*}
where $\nabla_\text{PD} \widehat{\cal L}(\pi)$ is the pathwise derivative estimate (the \emph{off-policy AGRO} gradient applied to $\mu=\pi$)
\begin{align}\label{eq:gradient.LR}
 & - \frac {\beta }{n}\sum_{i=1}^n \left( R_{\beta}^{\pi}(x,y_i) - \widehat R_{\beta, -i}^{\pi}(x,\pi) \right)\nabla\log\pi(y_i|x)
\end{align}
and $\nabla_\text{LR} \widehat {\cal L}(\pi)$ is the likelihood ratio gradient estimate:
\begin{align*}
\frac 1{2n}\sum_{i=1}^n \left( R_{\beta}^{\pi}(x, y_i) - \widehat R_{\beta, -i}^{\pi}(x,\pi) \right)^2 \nabla \log\pi(y_i|x).
\end{align*}
Thus, the full on-policy AGRO gradient $\nabla \widehat {\cal L}(\pi)$ can be written as
\begin{align}\label{eq:full.gradient.AGRO}
\frac 1{2n}\sum_{i=1}^n \left( R_{\beta}^{\pi}(x, y_i) - \widehat R_{\beta, -i}^{\pi}(x,\pi) -\beta \right)^2 \nabla \log\pi(y_i|x).
\end{align}

\begin{proposition} (\textbf{Unbiased on-policy AGRO gradient estimate}) \label{prop:unbiased-on-policy-grad}
$\nabla \widehat{\cal L}(\pi)$ is an unbiased estimate of $\nabla {\cal L}(\pi)$.
\end{proposition}

We postpone the full proof of theoretical results to Appendix~\ref{appendix:additional-theory}. We now illustrate a few properties of the gradient estimate.

\paragraph{Variance reduction.} Notice that we can build the estimate 
\begin{align*}
\nabla_\text{LR}^{\text{biased}} \widehat {\cal L}(\pi)\eqdef \frac 1{2n}\sum_{i=1}^n \left( A_i - \overline{A} \right) \nabla \log\pi(y_i|x),
\end{align*}
where $A_i\eqdef \left( R_{\beta}^{\pi}(x,y_i) - \widehat R_{\beta, -i}^{\pi}(x,\pi) \right)^2$, and $\overline{A}\eqdef \frac 1{n-1}\sum_{j\neq i} A_j$ is a leave-one-out baseline. This estimates may have a lower variance than $\nabla_\text{LR} \widehat {\cal L}(\pi)$ but at the price of introducing some bias. The bias stems from the fact that the baseline $\frac 1{n-1}\sum_{j\neq i} A_j$ now depends on the sample $y_i$, as opposed to typical regular leave-one-out applications \citep{kool2019buy,ahmadian2024,shao2024deepseekmath}.

\paragraph{Relation to on-policy RL.}
The usual approach in on-policy RL consists in maximizing the original objective $\pi\mapsto {\cal G}(\pi)$ by following the regularized policy gradient:
$$ \nabla \widehat {\cal G}(\pi) \eqdef \frac 1n\sum_{i=1}^n \left( r_i - \beta \log \frac{\pi(y_i|x)}{\piref(y_i|x)} \right)\nabla\log\pi(y_i|x),
$$
Notice that, similarly to Eq.~\eqref{eq:gradient.LR}, the leave-one-out baseline $\widehat R_{\beta, -i}^{\pi}(x,\pi)$ can be subtracted from the regularized reward in the expression above without introducing any bias (this is called the RLOO algorithm in \cite{ahmadian2024}).

We see that though the losses $\pi\mapsto {\cal G}(\pi)$ and $\pi\mapsto {\cal L}(\pi)$ have the same global optimum $\pi^*$, the on-policy AGRO gradient $\nabla \widehat {\cal L}(\pi)$ is not aligned with the usual on-policy RL gradient $\nabla \widehat {\cal G}(\pi)$. Actually, $\nabla \widehat {\cal L}(\pi)$ is the sum of two terms, $\nabla_\text{PD} \widehat {\cal L}(\pi)$, which is aligned with $\nabla \widehat {\cal G}(\pi)$, and $\nabla_\text{LR} \widehat {\cal L}(\pi)$, which intends to reduce the variance $\pi\mapsto {\mathbb V}_{y\sim \pi(\cdot|x)}\left( R_\beta^{\pi'}(x,y)\right)|_{\pi'=\pi}$. We will empirically investigate the impact of this extra gradient term in Section~\ref{sec:exp}.

\section{Implementation of AGRO at the token level}
\label{sec:token.level.gradient.estimate}

In the previous sections, we have expressed the gradient of our losses using sequence-level generations $y\sim \pi(\cdot|x)$. Henceforth for simplicity, we omit the dependency on $x$ for simplicity when the context is clear.
For LLM applications, the response $y$ is  produced as a sequence of tokens denoted as $(y_t)_{t=1}^T$ where $T$ is the sequence length. These tokens are generated auto-regressively $y_t\sim \pi(\cdot|y_{1:t-1})$, such that with the chain rule $\pi(y)=\pi(y_1) \pi(y_2|y_1)\dots \pi(y_T|y_{1:T-1})$. Using simplified notations, we write $\pi(y)=\pi(y_1) \pi(y_2)\dots \pi(y_T)$ by dropping the conditional dependency $\pi_t \coloneqq \pi(\cdot|y_{1:t-1})$.

Now, we show that we can use this specific token-level form to derive lower-variance gradient estimates. In particular, we focus on the gradient term related to the regularization
$$g(\pi) \eqdef \E_{y\sim \pi} \left[ \log\frac{\pi(y)}{\piref(y)}  \nabla\log\pi(y) \right].$$
We can build three alternative stochastic estimates of $g(\pi)$, given a sampled sequence $y\sim \pi$:
\begin{eqnarray*}
\hat g_1 \!\!\!\!&=& \!\!\!\! \sum_{t= 1}^T\sum_{t' = 1}^T \nabla\log\pi(y_t) \log\frac{\pi(y_{t'})}{\piref(y_{t'})} \\
\hat g_2 \!\!\!\!&=& \!\!\!\!\sum_{t= 1}^T\sum_{t' = t}^T \nabla\log\pi(y_t) \log\frac{\pi(y_{t'})}{\piref(y_{t'})} \\
\hat g_3 \!\!\!\!&=& \!\!\!\! \sum_{t=1}^T \nabla\log\pi(y_t)\Big( \log\frac{\pi(y_t)}{\piref(y_t)}  + \!\! \sum_{t'=t+1}^T \!\! \KL\left(\pi_{t'}, \pi_{\text{ref}, t'}\right) \! \Big) \! .
\end{eqnarray*}
These three estimates are  all unbiased. In fact, $\hat g_1$ is just the vanilla estimate of $g(\pi)$ by expanding $\pi(y)$ in its sequence form. In practice, such an estimate can be implemented by differentiating the squared loss function $
    \left(\log \frac{\pi(y)}{\piref(y)}\right)^2$, which can be readily computed along the sequence $y$.
    
Next, we discuss how $\hat g_2$ and $\hat g_3$ produces variance reduction with certain adjustments to the above implementation.

\paragraph{Variance reduction.}
We have constructed $\hat g_2$ to ignore the \emph{lower-triangular} terms of $\hat g_1$, which in expectation are zero. Indeed, for $t'<t$, the term $\E_{y\sim \pi} \left[ \nabla\log\pi(y_t) \log\frac{\pi(y_{t'})}{\piref(y_{t'})} \right]$ evaluates to zero 
$$\E_{y_{1:t-1} \sim \pi} \left[ \log\frac{\pi(y_{t'})}{\piref(y_{t'})}  \underbrace{\E_{y_{t}\sim \pi}\left[ \nabla\log\pi(y_t) \; | \; y_{1:t-1} \right]}_{=0} \right] = 0.$$
By removing these zero-mean terms from, $\hat g_2$ usually reduces the variance compared to $\hat g_1$. Finally, $\hat g_3$ further reduce the variance of $\hat g_2$ by explicitly computing the conditional expectation over each single next token (note this information is available since the distribution of next tokens in the forward pass). Indeed, for $t'>t$, the cross term $\E_{y\sim \pi} \left[ \nabla\log\pi(y_t) \log\frac{\pi(y_{t'})}{\piref(y_{t'})}  \right]$ evaluates as follows
\begin{align*}
& \E_{y_{1:t'-1} \sim \pi} \left[ \nabla\log\pi(y_t) \E_{y_{t'} \sim \pi} \left[ 
 \log\frac{\pi(y_{t'})}{\piref(y_{t'})} \Big| y_{1:t'-1}\right]\right] \\
&= \E_{y_{1:t'-1} \sim \pi} \left[ \nabla\log\pi(y_t)\KL(\pi_{t'}, \pi_{\text{ref}, t'})\right],
\end{align*}
where $\KL \left(\pi_{t'}, \pi_{\text{ref}, t'}\right)$ denotes the token-level KL divergence between $\pi(\cdot|y_{t'-1})$ and $\piref(\cdot|y_{t'-1})$ that can be computed with a summation over the space of tokens. As such, we can define several variants of the AGRO algorithm by implementing the gradient of the losses, using either the naive gradient estimate $\hat g_1$. We can also resort to variance reduction techniques that remove the lower-triangular terms ($\hat g_2$) or compute local expectations ($\hat g_3$).

Throughout the experiments, we have applied the vanilla estimate for AGRO as we see limited improvements due to variance reduction on small scale training. Nevertheless, the impact to expensive large-scale training can still be immense, which we leave to practitioners for further ablations.

\paragraph{Extension to the off-policy case.}
Notice that in the off-policy case, a similar variance reduction technique can be applied to the off-policy case $\nabla{\cal L}_{\mu}(\pi)$ by estimating the term $\E_{y\sim \mu} \left[ \log\frac{\pi(y)}{\piref(y)}  \nabla\log\pi(y) \right]$ with the variance-reduced estimate (akin to how $\hat g_2$ improves over $\hat g_1$)
$$\sum_{t=1}^T \nabla\log\pi(y_t) \sum_{t'=t}^T \log\frac{\pi(y_t)}{\piref(y_t)}.$$

For interested readers, we note that naive implementation of $\nabla_\text{LR}{\cal L}$ can also benefit from variance reduction, despite in a rather sophisticated manner. We provide more extended discussions in Appendix~\ref{apx:further.variance.reduction}.

Much of the discussion here extends to fine-tuning algorithms in prior work \citep{rafailov2024direct,azar2024general,calandriello2024human,richemond2024offline}, where sequence level algorithm requires careful adaptation to the token level implementation.

\section{Discussions on related work and extensions}

We further discuss how the AGRO formulation can be extended and how it relates to prior work.

\paragraph{Pairwise preference.} A dominant paradigm in RLHF is where the reward $r(x,y)$ is derived from a pairwise preference dataset, usually modeled using the Bradley-Terry (BT) assumption \citep{bradley1952rank}: $p(y\succ y'|x)=\exp\left(r(x,y)-r(x,y')\right)$. A direct implication of generation consistency from Theorem~\ref{thm:consistency} is that for any $y,y'\in Y$,
$$r(x,y)-r(x,y') = \beta\log\frac{\pi(y|x)}{\piref(y|x)} -\beta\log\frac{\pi(y'|x)}{\piref(y'|x)},$$ which when combined the BT loss produces the direct preference optimization loss \citep{rafailov2024direct}. Similar arguments can be extended to other generic loss functions for pairwise preference data \citep{zhao2023slic,azar2024general,tang2024generalized}.

\paragraph{Connections to regularized value learning.} The generation consistency (Theorem~\ref{thm:consistency}) can be understood as a special case of the general consistency for regularized value based learning in the RL  \citep{schulman2017,nachum2017bridging} and Generative Flow Networks literature \citep{bengio2023gflownet,mohammadpour2024maximum}. Indeed, if we interpret the log ratio $\log\frac{\pi(y|x)}{\piref(y|x)}$ as a Q-function, then the policy optimization algorithm is understood as a value learning algorithm. At the token level, the consistency condition can be written for any steps $t,t'$:
\begin{align*}
   \widetilde{V}^{\pi^*}(y_{1:t}) = \sum_{s=t}^{t'-1} \left( r_s - \beta \log \frac{\pi^*(y_s)}{\piref(y_s)}\right) + \widetilde{V}^{\pi^*}(y_{1:t'}),
\end{align*}
where $\widetilde{V}^{\pi^*}(y_{1:t})$ is the optimal regularized value function that depends on the whole partial sequence and $\pi^*$ the optimal regularized policy. Here, we have omitted the dependency on the prompt $x$ and $r_s$ is the per-token reward. Such a formulation might be useful in cases where dense rewards are used \citep{uesato2022solving,lightman2023let}.

Previously, we have established that the gradient of the off-policy AGRO with on-policy sampling was aligned with the gradient as the RLHF problem. Indeed, viewing AGRO as a value-based learning algorithm, this connection can be viewed as a special case of the equivalence between soft Q-learning and regularized policy optimization \citep{schulman2017,nachum2017bridging,o2022connection}.

\section{Experiments}\label{sec:exp}

\label{sec:exp}

Throughout, we focus on the mathematical reasoning dataset MATH \citep{hendrycks2021measuring} where the prompt $x$ consists of asking the model a mathematical question with a short-form ground truth answer $a^*$ available. Given the model generation $y=(c,a)$ which typically consists of a step-by-step chain-of-thought reasoning $c$ and a proposed answer $a$, the reward is computed as a match between $a$ and $a^\ast$. We adopt Sympy \citep{10.7717/peerj-cs.103} to automatically match the answers and assign a reward of $r=1$ if there is a match and $r=0$ otherwise. 

We fientune models on the MATH training set, which defines the prompt distribution $\rho$, with various objective alternatives introduced above. Throughout, we provide the model with a system prompt that asks for a step-by-step solution, followed by a final answer. Our experiments are based on the 8B model from the Llama 3 model family \citep{dubey2024llama}. All algorithmic variants apply identical hyper-parameters such as learning rate, and that they all apply $n=4$ samples for gradient estimations, which we detail in Appendix~\ref{appendix:hypers}.

\begin{figure}[t]
\centering
\includegraphics[width=3in]{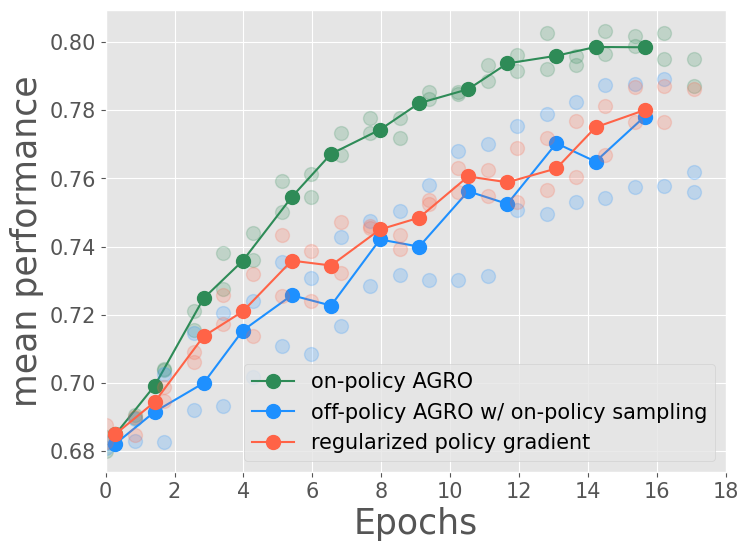}
\caption{\small{Training performance for on-policy learning. We compare three algorithmic alternatives: regularized policy gradient algorithm (red), off-policy AGRO algorithm with on-policy sampling (blue) and on-policy AGRO algorithm (green), all with regularization $\beta=0.001$. The performance is evaluated against the learning iterations, with each iteration being $100$ gradient updates. We observe a similar performance from the regularized policy gradient algorithm and off-policy AGRO, which is expected; on-policy AGRO seems to derive better performance over other baselines, given the same number of learning steps.}}
\label{figure:onpolicy-steps}
\end{figure}

\subsection{On-policy learning}

We start with the on-policy learning where we compare three algorithmic variants: regularized policy gradient algorithm (with leave-one-out baseline, i.e., RLOO \citep{ahmadian2024}), off-policy AGRO algorithm with on-policy sampling, and on-policy AGRO algorithm. For the base experiment, we apply a regularization coefficient of $\beta=0.001$.

Figure~\ref{figure:onpolicy-steps} shows the training performance as the training progresses. We ran each algorithmic variant with 2 seeds and show the average. A few observations are in order: (1) Note that in the case of perfect on-policy sampling, the off-policy AGRO algorithm (for $\mu=\pi$) should behave identically as the regularized policy gradient algorithm. We see that this is indeed the case within statistical errors across independent runs; (2) The on-policy AGRO algorithm seems more data-efficient, producing faster progress on the reward given same number of epochs on the training set. All algorithms can be implemented with similar computational cost per iteration, and as a result, the number of epochs is an approximation to the total computational cost. This means that the on-policy AGRO algorithm is also more data-efficient compared to alternatives - it requires less compute to achieve the same level of training performance.

Figure~\ref{figure:online-evals} shows the evaluation performance during the course of training. The results are fairly correlated with the training rewards - on-policy ARGO achieves the best performance overall, with a $+7\%$ lift in performance compared to $\piref$, while other baselines has a lift of $+4\%$.

\begin{figure}[t]
\centering
\includegraphics[width=3in]{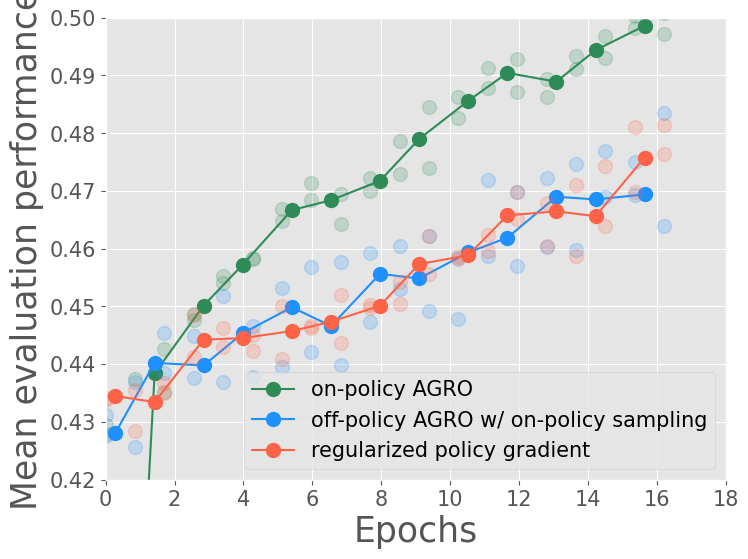}
\caption{\small{Evaluation on the test set for the training experiments conducted in Figure~\ref{figure:onpolicy-steps}. We plot the evaluation performance as a function of the learning iterations. We see that the training and evaluation performance are quite correlated, and that on-policy AGRO seems to achieve the best performance across different algorithmic variants.}}
\label{figure:online-evals}
\end{figure}

\subsection{Off-policy and offline learning}

We emulate the extreme case of off-policy learning where the sampling distribution is from the reference policy $\piref$ and fixed throughout training. We consider a few algorithmic baselines: AGRO and regularized policy gradient. To adapt the regularized policy gradient to the offline case, we consider the per-token KL regularization: given a trajectory $y$, the regularization is computed as $
    \sum_{t=1}^T \mathbb{KL}(\pi_t, \pi_\text{ref,t})$ 
which generally differs from the sequence-level KL regularization in the off-policy case \citep{calandriello2024human,tang2024generalized}. 

Figure~\ref{figure:offline} shows the evaluation performance of off-policy AGRO when compared against the baseline of a vanilla regularized policy gradient. AGRO and regularized policy gradient both apply a regularization of $\beta=0.01$. We make a few observations: (1) When regularization is weak, as with the case of unregularized policy gradient, the performance crashes as we go through more epochs on the training set. Such drop in performance does not happen for the on-policy case, which illustrates the potential instability of off-policy learning; (2) Off-policy AGRO seems to obtain better performance overall compared to KL-regularized policy gradient algorithm, obtaining both better peak performance and better performance at a fixed number of epochs. 

\begin{figure}[t]
\centering
\includegraphics[width=3in]{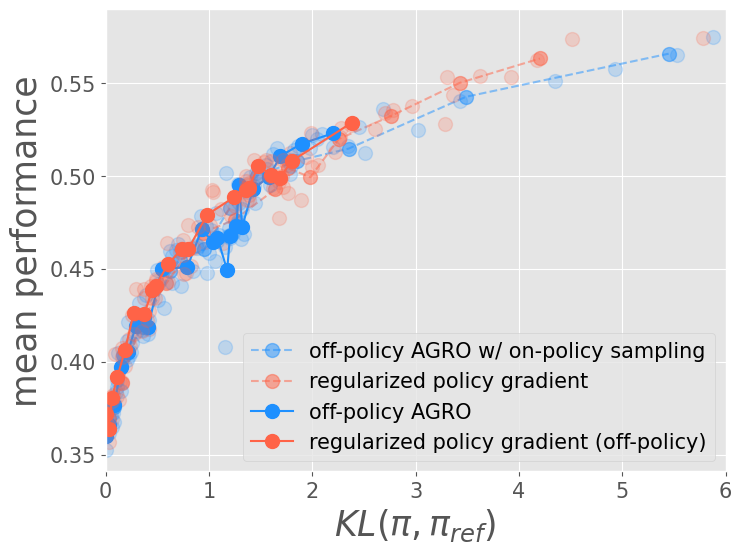}
\caption{Comparison of on--policy vs off-policy algorithms. Within the same number of updates, on-policy algorithms generally deviate more from the reference policy $\piref$, leading to larger increase in training performance. However, the KL-performance curve traced out by the off-policy algorithms aligns with the on-policy algorithms, indicating that we do not suffer any KL inefficiency despite being off-policy.}
\label{figure:offpolicy}
\end{figure}

\subsection{Ablation and comparison}

We have found that the difference between off-policy AGRO and KL-regularized policy gradient seems to enlarge as the level of off-policyness increases. Recall that the two algorithms are effectively equivalent when the sampling is fully on-policy, and their difference in regularization increases in the off-policy case. Below, we show further results on comparing the on-policy and off-policy performance: we see that though off-policy learning provides overall less improvement on performance, it can be equally KL-efficient as on-policy.

\paragraph{KL-efficiency comparison of on-policy and off-policy.}

We consider a synthetic off-policy learning setting where we keep a buffer of samples seen during on-policy sampling, and replay the the data at a probability $p=0.5$ to randomly replace the on-policy samplesNote that conceptually at the extreme when $p=0$ where no on-policy sample is replayed, the off-policyness is made extreme and we recover the offline setting.

As before, Figure~\ref{figure:offpolicy} shows the training performance as the training progresses. We plot against the KL divergence and compare the off-policy performance against the on-policy method. We make a few observations: (1) The off-policy learning is as KL-efficient as on-policy learning, though within the same amount of learning steps, on-policy learning drives up the KL divergence more, hence making faster progress to the training performance; 
(2) The on-policy algorithm converges faster when measured against the number of learning steps, likely due to the fact that the on-policy update tends to drive the learner policy further away from the reference policy. This corroborates some of the observations in prior work \citep{gao2023scaling,rame2024warp} where similar trade-offs have been investigated. (3) In this synthetic setting, there is not significant difference between the off-policy AGRO algorithm and regularized policy gradient, despite subtle difference in the regularization loss.

\begin{figure}[t]
\centering
\includegraphics[width=3in]{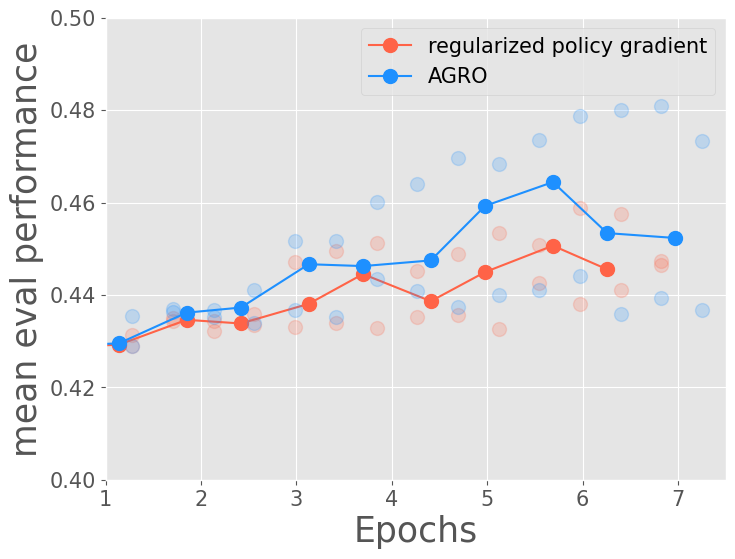}
\caption{\small{Evaluation performance for off-policy (offline) learning. We compare two algorithms: KL-regularized policy gradient (red) and off-policy AGRO (blue). We see that as training progresses, AGRO seems to obtain better overall performance than KL-regularized policy gradient.}}
\label{figure:offline}
\end{figure}

\paragraph{Ablation on regularization coefficient $\beta$.} In the on-policy experiments, we have varied the coefficient for both the regularized policy gradient algorithm $\beta\in\{0,0.0001,0.001\}$ and the off-policy AGRO algorithm $\beta\in\{0.0001,0.001,0.01,0.1\}$. The ablation results are presented in Figure~\ref{figure:ablation}.
In an expected way, strong regularization tends to slow down the policy optimization progress, i.e., with the same number of learning steps, the performance improvement is much slower at a large value of $\beta$. 

However, strong regularization makes the algorithm more KL-efficient (as similarly noted in prior work \citep{gao2023scaling,rame2024warp}). Intuitively, this is because a strong regularization makes the update much more conservative and this prevents the policy from deviating too much from the reference policy, while making progress on the training reward. For completeness, we also ablate the case where the optimization is unregularized ($\beta=0$), 
there is no incentive in preventing the policy from staying close to the reference policy. We find that no regularization makes the algorithm generally much more KL-inefficient and causes severe instability during off-policy learning. 

We finally comment on the choice of a good default value for $\beta$ that achieves a good balance between regularization and the default reward. For the squared regularization, note that its gradient scales as $\mathcal{O}(T)$. As a result, in order to have an effective regularization at a constant scale, we require $\beta=\mathcal{O}(1/T)$. We also note that if the updates are per token rather than per sequence \citep{grinsztajn2024averaging}, the scale of $\beta$ will be roughly $\mathcal{O}(1)$ and the hyper-parameter tuning process might be simpler in general.

\begin{figure}[t]
\centering
\includegraphics[width=3in]{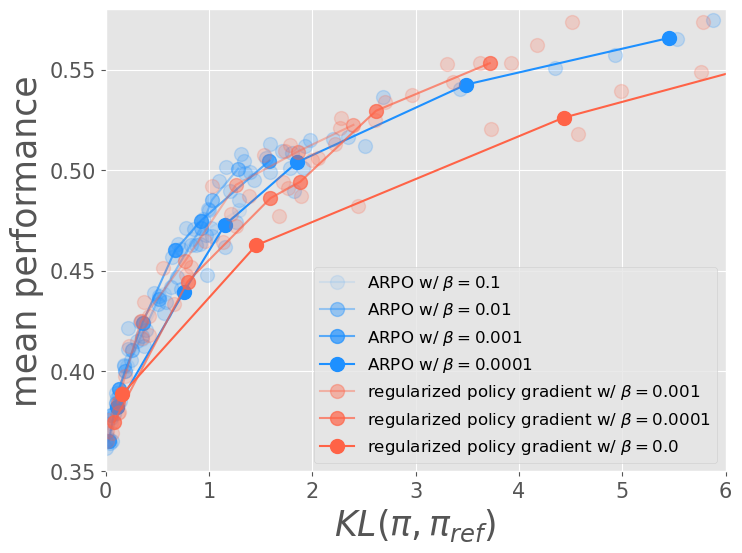}
\caption{\small{Ablation on the regularization coefficient $\beta$. We see that as $\beta$ increases, all algorithms generally become more KL-efficient as they are more conservative with the updates. However, they also tend to deviate less from the origin reference policy $\piref$ with a fixed number of updates.}}
\label{figure:ablation}
\end{figure}

\section{Discussion and limitation}

We have proposed AGRO, a fine-tuning algorithm leveraging both on-policy and off-policy data for regularized policy optimization. AGRO is based on generation consistency, an optimality condition for regularized policy optimization. This optimality condition allows for alternative loss functions that can naturally make use of any data sources. As a result, AGRO enjoys theoretical guarantees and produces compelling performance improvements over baseline regularized policy gradient algorithm in a few on-policy and off-policy learning settings.

We also note that the off-policy learning enabled by AGRO is not omnipotent and that learning can still be unstable due to off-policy data generation when taken to the extreme. As a result, potential future study includes a more careful investigation into the source of off-policy instability and how one might leverage complementary techniques such as importance sampling.

\bibliography{your_bib_file}
\bibliographystyle{plainnat}

\clearpage
\onecolumn

\begin{appendix}

\section*{\centering APPENDICES: RL-finetuning LLMs from on- and off-policy data with a single algorithm}

\section{Experimental details}\label{appendix:hypers}

We experimented with the Llama 3 model of size 8B. All experiments are conducted with identical hyper-parameter settings: we always apply a batch size of $B=64$ prompts per update, and sample $n=4$ distinct generations per prompt. All training and evaluation sampling are conducted at a temperature of $\tau=1$ and with $\text{top-p}=1$.

We train on the MATH training set with $7500$ examples and evaluate on the test set with $5000$ examples. A supervised fine-tuning on the training set is conducted to warm up the RL training, hence the gap between training and test set accuracy.

For both training and evaluation, we provide system instructions that ask the model to generate a response with step-by-step solution, followed by a final conclusion phrased as \emph{the final answer is} followed by the answer predicted by the model. This is consistent with the prompt structure discussed for Llama models \citep{dubey2024llama}. The reward is calculated by parsing the string after marker \emph{the final answer is} and matching against a ground truth answer that comes with the dataset.

All experiments are conducted with an KL regularization coefficient $\beta>0$ which we have ablated in the main paper.

\section{Additional theoretical results}
\label{appendix:additional-theory}
We provide extended discussion on additional theoretical results.

\subsection{Proof of theoretical results} 
We detail the proof of a number of theoretical results in the main paper.
\label{sec:appx.proofs}
\paragraph{Proof of Lemma~\ref{lemma:off-policy-grad}.}
\begin{proof}
The proof follows from taking the derivative of a squared loss using the fact that $\nabla R^\pi_\beta(x,y) = -\beta \nabla\log\pi(y|x)$, and that $\E_{y\sim\mu(\cdot|x)} \left[ \left( 
R^\pi_\beta(x,y) - R^\pi_\beta(x,\mu)
\right)
\nabla R^\pi_\beta(x,\mu)
\right]  = 0.
$ This concludes the proof.
\end{proof}

\paragraph{Proof of Proposition~\ref{prop:unbiased-on-policy-grad}.}

\begin{proof}
From the off-policy AGRO gradient we have $\E_{x\sim\rho, y_1,\dots,y_n\sim \pi(\cdot|x)}\left[ \nabla_\text{PD} \widehat {\cal L}(\pi)\right]
 = \nabla_\text{PD} {\cal L}(\pi)$. 
Now again, using that $y_1,\dots, y_n$ are i.i.d.~and that $\widehat R_{\beta, -i}^{\pi}(x,\pi)$ is independent of $y_i$ (and in expectation equals $R_{\beta}^{\pi}(x,\pi)$), we have that $\E_{x\sim\rho, y_1,\dots,y_n\sim \pi(\cdot|x)}\left[ \nabla_\text{LR} \widehat {\cal L}(\pi)\right]= \nabla_\text{LR} {\cal L}(\pi)$. Thus $\E\left[ \nabla \widehat{\cal L}(\pi) \right] = \nabla {\cal L}(\pi)$ and the proof is concluded.
\end{proof}

\subsection{Equivalence between KL divergence and performance under regularized reward}
We have hinted at the connection between the KL divergence $\mathbb{KL}(\pi,\pi^*)$ and the performance under regularized reward. We make the statement formal below, which is a special case of \citep{o2022connection}. Whenever the context is clear we drop the dependency on $x$.
\begin{lemma} (\textbf{Equivalence})
    The KL divergence $\mathbb{KL}(\pi,\pi^*)\eqdef \mathbb{E}_{x\sim\rho}\left[ \mathbb{KL}(\pi(\cdot|x),\pi^*(\cdot|x))\right]$ is related to the performance ${\cal G}(\pi)\eqdef \mathbb{E}_{x\sim\rho, y\sim \pi}\left[R_\beta^\pi(x,y)\right]$ under the regularized policy as 
    \begin{align*}
        \beta \mathbb{KL}(\pi,\pi^*) = {\cal G}(\pi^*) - {\cal G}(\pi).
    \end{align*}
\end{lemma}
\begin{proof}
    By the property of the optimal policy $\pi^*(y|x)=\piref(y|x)\exp(\beta^{-1}r(x,y))/\exp(\beta^{-1} \widetilde V^{\pi^*}(x))$, we have
    \begin{align*}
       \beta \mathbb{KL}(\pi,\pi^*) = \mathbb{E}_{x\sim\rho, y\sim \pi(\cdot|x) }\left[-r(x,y)  +\beta\log\frac{\pi(y|x)}{\piref(y|x)} + \widetilde V^{\pi^*}(x) \right] = -{\cal G}(\pi) + \mathbb{E}_{x\sim\rho}\left[\widetilde V^{\pi^*}(x)\right].
    \end{align*}
We conclude the proof by noticing that $\mathbb{E}_{x\sim\rho}\left[\widetilde V^{\pi^*}(x)\right]  =\mathbb{E}_{x\sim\rho}\left[V^{\pi^*}(x) - \beta \mathbb{KL}(\pi(\cdot|x),\pi^*(\cdot|x))\right] = {\cal G}(\pi^*)$.
\end{proof}

\subsection{Variance reduction for estimating $\nabla_\text{LR}{\cal L}(\pi)$}\label{apx:further.variance.reduction}

In Section~\ref{sec:token.level.gradient.estimate} we introduced token-level gradient estimates to reduce the variance of estimating the gradients $\nabla_\text{PD}{\cal L}$ and $\nabla {\cal G}$. Here we focus on reducing the variance of $\nabla_\text{LR}{\cal L}$. 

The cross product terms of the quadratic form that defines $\nabla_\text{LR}{\cal L}$ can be expressed as a function of $\E_{y\sim \pi} \left[ \log\frac{\pi(y)}{\piref(y)}  \nabla\log\pi(y) \right]$ thus can be handled using the estimates discussed in Section~\ref{sec:token.level.gradient.estimate}. Now we focus on reducing the variance of estimating the squared term 
$$g(\pi) \eqdef \E_{y\sim \pi} \left[ \left( \log\frac{\pi(y)}{\piref(y)} \right)^2 \nabla\log\pi(y) \right].$$ 

Expanding the above into the sequence form, we have
\begin{eqnarray*}
g(\pi) &=& \E_{y\sim \pi} \left[ \sum_{t=1}^T \sum_{t'=1}^T \sum_{s=1}^T \log\frac{\pi(y_t)}{\piref(y_t)} \log\frac{\pi(y_{t'})}{\piref(y_{t'})}  \nabla\log\pi(y_s) \right].
\end{eqnarray*}

The triple sum $\sum_{t=1}^T \sum_{t'=1}^T \sum_{s=1}^T$ can be decomposed into a few parts
\begin{itemize}
    \item The sum over terms $1\leq t, t'<s\leq T$ which are zero, since $\E_{y_s}\left[\nabla\log\pi(y_s|y_{1:s-1})\right]=0$.
    \item The sum over terms $1\leq s\leq t,t'\leq T$, which can be regrouped as $\sum_{s=1}^T \nabla\log\pi(y_s) \left( \sum_{t=s}^T \log\frac{\pi(y_t)}{\piref(y_t)} \right)^2 $.
    \item The sum over terms $1\leq t<s<t'\leq T$ and $1\leq t'<s<t\leq n$, which can benefit from computing explicitly one-step expectations (similarly as what has been done in Section~\ref{sec:token.level.gradient.estimate}), which gives 
    $$2 \sum_{t=1}^T \log\frac{\pi(y_t)}{\piref(y_t)} \sum_{s=t+1}^T \nabla\log\pi(y_s)\sum_{t'=s+1}^T \KL(\pi_{t'}, \pi_{\text{ref}, t'}).$$
\end{itemize}

Putting everything together, the low-variance estimate of $g(\pi)$ is
\begin{eqnarray*}
\hat g &=& \sum_{s=1}^T \nabla\log\pi(y_s)
\left[ \left( \sum_{t=s}^T \log\frac{\pi(y_t)}{\piref(y_t)} \right)^2  + 2 \sum_{t=1}^{s-1} \log\frac{\pi(y_t)}{\piref(y_t)} \sum_{t'=s+1}^T \KL(\pi_{t'}, \pi_{\text{ref}, t'}) \right].
\end{eqnarray*}
While such low-variance estimates are desirable in principle, they can be more sophisticated to implement when accounting for other practical implementation trade-offs.

\end{appendix}

\end{document}